\documentclass[runningheads]{llncs}
\usepackage{amsfonts}
\usepackage{amsmath}
\usepackage{mathrsfs}

\begin{document}

\title{Measuring the intelligence of an idealized mechanical knowing agent
}

\titlerunning{Measuring intelligence}

\author{Samuel Allen
Alexander\inst{1}\orcidID{0000-0002-7930-110X}}

\institute{The U.S.\ Securities and Exchange Commission
\email{samuelallenalexander@gmail.com}
\url{https://philpeople.org/profiles/samuel-alexander/publications}}

\maketitle

\begin{abstract}
  We define a notion of the intelligence level of an idealized
  mechanical knowing agent. This is motivated by efforts within
  artificial intelligence research to define real-number intelligence
  levels of complicated intelligent systems. Our agents are more
  idealized, which allows us to define
  a much simpler measure of intelligence level for them. In short, we define
  the intelligence level of
  a mechanical knowing agent to be the supremum of the computable
  ordinals that have codes the agent knows to be codes of
  computable ordinals. We prove
  that if one agent knows certain
  things about another agent, then the former necessarily has
  a higher intelligence level than the latter. This allows our
  intelligence notion to serve as a stepping stone to obtain
  results which, by themselves, are not stated in terms of our
  intelligence notion (results of potential interest even
  to readers totally skeptical that our notion correctly captures
  intelligence). As an application, we argue
  that these results comprise evidence against the possibility of
  intelligence explosion (that is, the notion that sufficiently intelligent
  machines will eventually be capable of designing even more intelligent
  machines, which can then design even more intelligent machines, and so
  on).
\end{abstract}

\section{Introduction}

In formal epistemology, when we study the knowledge of knowing
agents, we usually \emph{idealize} their knowledge. We assume,
for example, that if an agent knows $A$ and knows $A\rightarrow B$,
then that agent knows $B$. We might assume the agent knows all the
first-order axioms of Peano arithmetic, even though there are infinitely many
such axioms (because the axiom of mathematical induction is an infinite
schema). See \cite{shapiro1998} (section 2) for an excellent description of
this idealization process. This idealization process is important because it
acts as a
simplifying assumption which makes it possible to reason about knowledge.
Without such simplifying assumptions, the deep structure of knowledge would
be hidden behind the distracting noise and arbitrariness surrounding
real-world knowledge. In this paper, we will describe a way to measure
the intelligence level of an idealized mechanical knowing agent (a knowing agent
is \emph{mechanical} if its knowledge-set can be enumerated by a Turing machine,
see \cite{carlson2000}).

We anticipate that the reader might object 
that knowing agents might not
be totally ordered by intelligence (perhaps there are two agents $A$ and $B$ such
that $A$ is more intelligent in certain ways and $B$ is more intelligent in others);
the same goes for human beings, but that does not stop psychologists from studying
IQ scores. Our intelligence measure is somewhat like an IQ test in the sense
that it assigns intelligence levels to agents in
spite of the fact that true intelligence probably is not a total ordering.
Similarly, the reader might object that intelligence is not 1-dimensional and therefore
one single measurement is probably not enough. Again, we would make the same comparison
to IQ. In general, any formal measure of intelligence is certain to have limits
(the map is not the territory).

This paper was motivated by authors like
Legg and Hutter \cite{hutter2007},
Hern\'andez-Orallo and Dowe \cite{hernandez},
and Hibbard \cite{hibbard},
who attempt to use real numbers
to measure the intelligence of intelligent systems\footnote{In the case of
Hibbard, natural numbers are used.}. Those systems perform actions and observe
the results of those
actions in surrounding environments\footnote{Such authors essentially consider an
\emph{environment} to be a function which takes as input a finite sequence of \emph{actions}
and which outputs a real-number \emph{reward} and an \emph{observation} for each such
action-sequence.
To those authors, an \emph{intelligent system} is essentially a function which takes
a finite sequence of
reward-observation pairs and outputs an \emph{action}.
A system and an environment interact with each other to produce
an infinite reward-observation-action sequence in the obvious way.
Those authors' goal is to assign numerical intelligence-measurements to such systems,
with the intention that a higher-intelligence system should outperform a lower-intelligence
system (as measured by total reward earned) ``on average'' (across an infinite universe
of environments). This is, of course, an oversimplification of those authors' work.}.
By contrast, the agents we consider
do not perform actions in their environments, neither do they make observations
about those environments. To us, a knowing agent is more like an
intelligent system that has been placed in a particularly bleak environment:
an empty room totally devoid of stimulus and rewards.
Thus abandoned, the
system has nothing else to do but enumerate theorems all day. Despite these
spartan conditions, we discovered a
method of measuring the intelligence of idealized mechanical knowing agents.
(In Section \ref{thoughtexperimentsection}, we will describe a thought experiment
whereby our idealized agents can be obtained as a type of cross section of
less idealized agents, so that in spite of the idealized nature of the agents
we predominately study, nevertheless some insight can be gained into more realistic
intelligent systems.)

Whereas authors like Legg and Hutter attempt to measure intelligence based on
what an intelligent system \emph{does}, we measure intelligence based on
what a knowing mechanical agent \emph{knows}.
And whereas authors like Legg and Hutter use real numbers to measure intelligence,
our method uses computable ordinal numbers instead.

To see one of the benefits of using ordinals to measure intelligence, consider the
following question: if $A_1,A_2,\ldots$ are agents such that each $A_{i+1}$ is
significantly more intelligent than $A_i$, does it necessarily follow that for every agent
$B$, there must be some $i$ such that $A_i$ is more intelligent than $B$?
If we were to measure intelligence using natural numbers (for example, as the
Kolmogorov complexity of the agent), the answer would automatically be ``yes'',
but the reason has nothing to do with intelligence and everything to do with
the topology of the natural numbers. A real-number-valued intelligence
measure would also force the answer to be ``yes'' assuming that ``$A_{i+1}$ is significantly
more intelligent than $A_i$'' implies ``$A_{i+1}$'s intelligence is at least $+1$
higher than $A_i$'s intelligence''. In actuality, we see no reason why the answer
to the question must be ``yes'', at least for idealized agents. Imagine
a master-agent who designs sub-agents. Over the course of eternity, the master-agent
might design better and better sub-agents, each one significantly more intelligent
than the previous, but each one intentionally kept less intelligent than the
master-agent.

Another benefit of measuring intelligence using computable ordinals is
that, because the computable ordinals are well-founded (i.e., there is no
infinite strictly-descending sequence of computable ordinals), we
obtain a well-founded structure on idealized knowing agents (i.e., there is
no infinite sequence of mechanical knowing agents each with strictly greater
intelligence than the next). Further, this well-foundedness is inherited by any
relation
on idealized mechanical knowing agents that respects our intelligence measure
(a relation $\prec$ is said to \emph{respect} an intelligence measure if
whenever $B\prec A$, then $A$ has a higher intelligence than $B$ according to that measure).
For example, say that knower $A$ \emph{totally endorses} knower $B$ if
$A$ knows the codes of Turing machines that enumerate $B$ and also $A$ knows that
$B$ is truthful (we will better formalize this later). We will show that whenever
$A$ totally endorses $B$, $A$ has a strictly larger intelligence than $B$ according
to our ordinal-valued measure of intelligence. It immediately follows that there is
no infinite sequence of mechanical knowing agents, each one of which totally endorses
the next.
(A result which, although we arrive at it by means of our intelligence measure, does
not itself make any direct reference to our intelligence measure, and should be of
interest even to critics who would flatly deny that our intelligence measure is the
correct way to measure a mechanical knowing agent's intelligence.)

As a practical application, the result in the previous paragraph provides a
skeptical lens through which to view the idea of intelligence explosion,
as described by Hutter \cite{hutter2012}. We will elaborate upon this
in Section \ref{intelligenceExplosionSection}.

\section{An intuitive ordinal notation system}
\label{intuitivesection}

Whatever intelligence is, it surely involves
certain core components like: pattern-matching; creativity; and the ability
to generalize. In this section we introduce an intuitive ordinal notation system which
will illuminate the relationship between ordinal notation and those three
core components of intelligence. Later in the paper, in order to simplify
technical details, we will use an equivalent but more abstract ordinal notation
system.

\begin{definition}
\label{programsasnotations}
  Let $\mathscr P$ be the smallest set of computer programs such that for every
  computer program $P$, if,
  when $P$ is run,
  all the outputs of $P$ are elements
  of $\mathscr P$, then $P\in\mathscr P$.
  For each $P\in\mathscr P$, let $|P|$ be the smallest ordinal
  $\alpha$ such that $\alpha>|Q|$ for every program $Q$ which $P$ outputs.
  We say that $P$ \emph{notates} the ordinal $|P|$.
\end{definition}

\begin{example}
\label{somefiniteords}
  (Some finite ordinals)
  \begin{enumerate}
  \item
  Let $P_0$ be the program ``End'', which immediately ends, outputting nothing.
  Vacuously, $P_0$ outputs nothing except elements of $\mathscr P$, so $P_0\in\mathscr P$.
  $|P_0|$ is the smallest ordinal $\alpha$ bigger than $|Q|$ for every $Q$ which $P_0$
  outputs, so $|P_0|=0$ (since $P_0$ outputs nothing).
  \item
  Let $P_1$ be the program: ``Print(`End')'', which outputs $P_0$ and then stops.
  Certainly $P_1$ outputs nothing except elements of $\mathscr P$, so $P_1\in\mathscr P$.
  $|P_1|$ is the smallest ordinal $\alpha$ bigger than $|Q|$ for every $Q$ which $P_1$
  outputs, so $|P_1|=1$.
  \item
  Let $P_2$ be the program:
  ``Print(`Print(`End')')'', which outputs
  $P_1$ and then stops. Then $P_2\in\mathscr P$ and $|P_2|=2$.
  \end{enumerate}
\end{example}

Using their pattern-matching skills, the reader should recognize a
pattern forming in Example \ref{somefiniteords}. Through the use of creativity and
generalization, the reader can short-circuit that pattern to
obtain the first infinite ordinal, $\omega$.

\begin{example}
\label{omegaexample}
  Let $P_\omega$ be the program:
  \[\mbox{
    Let X $=$ `End'; While(True) \{ Print(X); Let X $=$ ``Print(`''+X+``')'' \}
  }\]
  which outputs ``End'', ``Print(`End')'',
  ``Print(`Print(`End')')'', and so on forever.
  By reasoning similar to Example \ref{somefiniteords}, these outputs are in
  $\mathscr P$ and they notate $0,1,2,\ldots$. Thus $P_\omega\in\mathscr P$
  and $|P_\omega|$ is the smallest ordinal bigger than all of $0,1,2,\ldots$,
  i.e., the smallest infinite ordinal, $\omega$.
\end{example}

One might think of $P_\omega$ as a naive attempt to print every ordinal.
The attempt fails, of course, because it does not print $P_\omega$ itself. In similar
fashion, it can be shown that no program can succeed at printing exactly the
set of computable ordinals ($\mathscr P$ is not computably enumerable).

\begin{example} (The next few ordinals)
\label{nextfewordsexample}
  \begin{enumerate}
  \item
  Let $P_{\omega+1}$ be the program: ``Print($P_\omega$)'' (where $P_\omega$
  is from Example \ref{omegaexample}). Then $P_{\omega+1}\in\mathscr P$
  and $|P_{\omega+1}|=\omega+1$.
  \item
  Let $P_{\omega+2}$ be: ``Print($P_{\omega+1}$)''.
  Then $P_{\omega+2}\in\mathscr P$ and $|P_{\omega+2}|=\omega+2$.
  \end{enumerate}
\end{example}

We could continue Example \ref{nextfewordsexample} all day, notating
$\omega+3$, $\omega+4$, and so on. But the reader is more intelligent than that.
Using their pattern-matching skill, their creativity, and their generalization skill,
the reader can short-circuit the process.

\begin{example}
\label{accelerating}
  (Starting to accelerate)
  \begin{enumerate}
  \item
  Let $P_{\omega\cdot 2}$ be the program:
  \[\mbox{
    Let X $=$ $P_\omega$; While(True) \{ Print(X); Let X $=$ ``Print(`''+X+``')'' \}
  }\]
  Similar to Example \ref{omegaexample}, $P_{\omega\cdot2}\in\mathscr P$
  and $|P_{\omega\cdot 2}|=\omega\cdot 2$.
  \item
  Let $P_{\omega\cdot3}$ be the program:
  \[\mbox{
    Let X $=$ $P_{\omega\cdot 2}$; While(True) \{ Print(X); Let X $=$ ``Print(`''+X+``')'' \}
  }\]
  Then $P_{\omega\cdot3}\in\mathscr P$ and $|P_{\omega\cdot3}|=\omega\cdot 3$.
  \end{enumerate}
\end{example}

Again, we could continue Example \ref{accelerating} all day,
notating $\omega\cdot4$, $\omega\cdot5$, and so on. But the reader is
more intelligent than that and can identify the pattern and creatively abstract it
to reach $\omega\cdot\omega=\omega^2$:

\begin{example}
\label{omegasquaredexample}
  Let $P_{\omega^2}$ be the program:
  \begin{quote}
    Let LEFT
      $= \ulcorner \mbox{Let X $=$ ``}\urcorner$;\\
    Let RIGHT $= \ulcorner \mbox{''; While(True)\{
      Print(X); Let X $=$ ``Print(`''+X+``')''
    \}}\urcorner$;\\
    Let X $= \ulcorner\mbox{End}\urcorner$;\\
    While(True) \{
    \begin{quote}
      Let X $=$ LEFT + X + RIGHT;\\
      Print(X)
    \end{quote}
    \}
  \end{quote}
  $P_{\omega^2}\in\mathscr P$ notates $|P_{\omega^2}|=\omega^2$.
\end{example}

We can continue along these same lines as long as we like, without ever reaching
an end:

\begin{exercise}
\label{ordinalexercise}
  \begin{enumerate}
    \item Write programs notating $\omega^3$, $\omega^4$, $\ldots$.
    \item Use your creativity and your pattern-matching and generalization skills to
      notate $\omega^\omega$.
    \item Write programs notating $\omega^{\omega^\omega}$, $\omega^{\omega^{\omega^\omega}}$,
      $\ldots$.
    \item Use your creativity and your pattern-matching and generalization skills to
      short-circuit the above and notate the smallest ordinal, called $\epsilon_0$,
      with the property that $\epsilon_0=\omega^{\epsilon_0}$.
    \item Contemplate creative ways to go far beyond $\epsilon_0$.
  \end{enumerate}
\end{exercise}

In the above examples and exercises, at various points we need to apply creativity
to transcend all the techniques developed previously.
I conjecture that each such transcending requires strictly greater intelligence
than the ones before it. If this informal conjecture is true,
then it seems natural to measure an intelligence by saying: an agent's
intelligence level is equal to the supremum of the ordinals the agent comes up
with if the agent is allowed to spend all eternity inventing ordinal
notations\footnote{For another connection to intelligence,
consider the open-ended problem: ``Find a very fast-growing computable function''. It
seems plausible that solutions should span much or all of
the range of mathematical intelligence. And yet,
so-called \emph{fast-growing hierarchies} (which ultimately trace back to
G.H.\ Hardy \cite{hardy}) essentially reduce the problem to that
of notating computable ordinals.}.

Theoretically, the above examples and exercises might someday be able to serve
as a bridge between artificial intelligence research and neuroscience. Namely: observe
human subjects' brains while they work on designing the programs in question, to see
how the magnitude of the ordinal being notated corresponds to the regions of the
brain that activate.

\section{Preliminaries}

In Section \ref{intuitivesection} we introduced an intuitive ordinal notation
system (and, implicitly, the notion of the \emph{output} of a computer program).
To get actual work done, we'll need an ordinal notation system (and notion of
program output) which is easier to work with. We begin with a formalized notion
of computer program outputs.

\begin{definition}
  For $n\in\mathbb N$, let $W_n$ be the $n$th computably enumerable
  set of natural numbers (i.e., the set of naturals enumerated by the $n$th Turing machine).
\end{definition}

For example, if $n$ is such that the $n$th Turing machine never halts, then
$W_n=\emptyset$. If $n$ is such that the $n$th Turing machine enumerates exactly the
prime numbers,
then $W_n$ is the set of prime numbers.

The following ordinal notation system is equivalent to Definition \ref{programsasnotations}
but easier to formally work with. This ordinal notation system is a
simplification of a well-known ordinal notation system invented by Kleene \cite{kleene}.

\begin{definition}
\label{kleeneO}
(Compare Definition \ref{programsasnotations})
Let $\mathcal O$ be the smallest subset of $\mathbb N$ with the property that
for every $n\in\mathbb N$, if $W_n\subseteq \mathcal O$, then $n\in\mathcal O$.
For each $n\in\mathcal O$, let $\left|n\right|$ be the smallest ordinal $\alpha$
such that $\alpha>\left|m\right|$ for all $m\in W_n$.
For $n\in\mathcal O$, we say that $n$ \emph{notates} $\left|n\right|$.
\end{definition}

Intuitively, we want to identify a knowing agent with its knowledge-set in a certain
carefully-chosen language.
The language will contain a symbol $\overline n$ for each natural number $n$;
a symbol $\mathbf W$ (we
intend that $\mathbf W(x,y)$ be read like ``$x\in W_y$'');
a symbol $\mathbf O$ (we intend that $\mathbf O(x)$
be read like ``$x\in\mathcal O$'');
and finally, the language will contain modal operators $K_1,K_2,\ldots$.
For any formula $\phi$ in the language, the formula $K_i(\phi)$ is intended
to express ``Agent $i$ knows $\phi$''.
When no confusion results, we will abbreviate $K_i(\phi)$ as $K_i\phi$.
For example, suppose we have chosen Agents $1,2,3,\ldots$. Agent (say) $5$ shall be
identified with the set of statements (within the language) that Agent 5 knows.
If one of the statements known by Agent 5 is $K_7(\overline 1=\overline 1)$, then that
statement is read like ``Agent $7$ knows $1=1$'', which is semantically interpreted as the
statement that Agent 7 (i.e., the set of Agent 7's knowledge) contains the statement
$\overline 1=\overline 1$.
Nontrivial statements can be built up using quantifiers. For example, the statement
$\forall x (K_2\mathbf O(x)\rightarrow K_3\mathbf O(x))$ expresses that for every
natural number $n$, if Agent 2 knows $n\in\mathcal O$, then Agent 3 also knows
$n\in\mathcal O$.

Unfortunately, the naive intuition in the above paragraph would expose us to
philosophical questions like ``what does it mean for a statement to be true?''
Thus, we must formalize everything using techniques from mathematical logic.
A reader uninterested in all the formal details can safely skim the definitions
in this section (which assume familiarity with first-order logic) and instead read
our commentary on those definitions.

\begin{definition}
\label{stddefn}
(Standard Definitions)
\begin{enumerate}
    \item When a first-order model $\mathscr M$ is clear from context, an
    \emph{assignment} is a function $s$ mapping the set of first-order variables into
    the universe of $\mathscr M$.
    \item For any assignment $s$, variable $x$, and element $u$ of
    the universe of $\mathscr M$,
    $s(x|u)$ is the assignment which agrees with $s$ everywhere except that it maps $x$ to $u$.
    \item For any variable $x$ and any formula $\phi$ and term $t$ in a first-order language,
    $\phi(x|t)$ is the result of substituting $t$ for $x$ in $\phi$.
    \item If $\mathscr M$ is a first-order model over a first-order language $\mathscr L$,
    and if $\phi$ is an $\mathscr L$-formula such that $\mathscr M\models \phi[s]$ for all
    assignments $s$, then we say $\mathscr M\models\phi$.
    \item An $\mathscr L$-formula $\phi$ is a \emph{sentence} if it has no free variables.
    \item An $\mathscr L$-formula $\phi$ is \emph{tautological} if for every $\mathscr L$-model
    $\mathscr M$, $\mathscr M\models \phi$.
    \item A \emph{universal closure} of a formula $\phi$ is a sentence
    $\forall x_1\cdots\forall x_n\phi$ where the variables $x_1,\ldots,x_n$ include all the free variables of $\phi$.
\end{enumerate}
\end{definition}

Definition \ref{stddefn} merely reviews standard material from first-order logic.
Informally, $\mathscr M\models\phi$
can be read, ``$\phi$ is true (in model $\mathscr M$)''.
First-order logic does not touch on modal operators like $K_i$, so we need to
extend first-order logic. We want to work with statements involving modal operators
and also quantifiers ($\forall$, $\exists$) in the same statement---we want to do
quantified modal logic. Quantified modal logic semantics is relatively cutting-edge.
For our extension, we will make use of the
so-called \emph{base logic} from \cite{carlson2000},
as rephrased in \cite{alexander2015}.

\begin{definition}
(The base logic)
\begin{itemize}
\item
A language $\mathscr L$ of the \emph{base logic} consists of a first-order language
$\mathscr L_0$ together with a set of symbols called \emph{operators}. $\mathscr L$-formulas
and their free variables are defined as usual, with the additional clause that for any
operator $K$ and any $\mathscr L$-formula $\phi$, $K(\phi)$ is an $\mathscr L$-formula,
with the same free variables as $\phi$. Syntactic parts
of Definition \ref{stddefn} extend to the base logic
in the obvious ways.
\item
With $\mathscr L$ as above, an $\mathscr L$-model $\mathscr M$ consists of a
first-order model $\mathscr M_0$ for $\mathscr L_0$, along with a function which takes one
operator $K$, one $\mathscr L$-formula $\phi$, and one $\mathscr M_0$-assignment $s$, and
outputs either True or False--in which case we write $\mathscr M\models K\phi[s]$ or
$\mathscr M\not\models K\phi[s]$, respectively--such that:
    \begin{enumerate}
        \item Whether or not $\mathscr M\models K\phi[s]$ does not depend on $s(x)$ if
        $x$ is not a free variable of $\phi$.
        \item Whenever $\phi$ and $\psi$ are alphabetic invariants (by which we mean that
        one is obtained from the other by renaming bound variables in a way which is consistent
        with the binding of the quantifiers), then $\mathscr M\models K\phi[s]$ if and only
        if $\mathscr M\models K\psi[s]$.
        \item For variables $x$ and $y$ such that $y$ is substitutable for $x$ in $K\phi$,
        $\mathscr M\models K\phi(x|y)[s]$ if and only if $\mathscr M\models K\phi[s(x|s(y))]$.
    \end{enumerate}
The definition of $\mathscr M\models \phi[s]$ (and of $\mathscr M\models\phi$) for
arbitrary $\mathscr L$-formulas $\phi$
is obtained from this by induction. Semantic parts of Definition \ref{stddefn} extend
to the base logic in the obvious ways.
\end{itemize}
\end{definition}

The following are some standard axioms which
any idealized knowing agent presumably should satisfy.
Axioms E1-E3 below are taken from \cite{carlson2000}.

\begin{definition}
\label{axiomsofknowledge}
Suppose $\mathscr L$ is a language in the base logic, with an operator $K$.
The \emph{axioms of knowledge} for $K$ in $\mathscr L$ consist of the
following schemas, where $\phi, \psi$ vary over
$\mathscr L$-formulas.
\begin{itemize}
    \item (E1) Any universal closure of $K\phi$ whenever $\phi$ is tautological.
    \item (E2) Any universal closure of
      $K(\phi\rightarrow\psi)\rightarrow K\phi\rightarrow K\psi$.
    \item (E3) Any universal closure of $K\phi\rightarrow \phi$.
\end{itemize}
By the \emph{axioms of knowledge} in $\mathscr L$, we mean the set of axioms of knowledge for
$K$ in $\mathscr L$, for all $\mathscr L$-operators $K$.
\end{definition}

For example, for operator $K_5$, the corresponding E3 schema expresses the
truthfulness of Agent 5, stating that whenever Agent 5 knows a fact $\phi$
(i.e., whenever $K_5\phi$ is true in the model in question), then $\phi$ is true
(in the model in question). The E1 schema for $K_5$ essentially states
that Agent 5 is smart enough to know tautologies. The E2 schema for $K_5$ expresses that
Agent 5's knowledge is closed under modus ponens: whenever Agent 5 knows
$\phi\rightarrow\psi$ and also knows $\phi$, then Agent 5 knows $\psi$.


We will now formally define the language we spoke of intuitively above.
The lack of the usual arithmetical symbols $S$, $+$ and $\cdot$ might be
surprising to mathematical logicians; we do not need those symbols.
Their absense emphasizes that our results
are independent of G\"odel-style diagonalization\footnote{To be clear,
our results would still apply to
agents who are aware of these arithmetical symbols, but our results do not
require as much. Our most important results concern well-foundedness,
which is a negative property (because it states a \emph{lack} of
infinite descending sequences), and so by weakening our language like this, we
strengthen those results.}.

\begin{definition}
\begin{itemize}
\item
Let $\mathscr L_O$ be the language which has a
constant symbol $\overline n$ for each $n\in\mathbb N$, a unary predicate symbol $\mathbf O$
(intended as a predicate for the set $\mathcal O$
of ordinal notations),
a binary predicate symbol $\mathbf W$ (we intend that $\mathbf W(x,y)$ be
interpreted as $x\in W_y$ where $W_y$ is the $y$th computably enumerable set),
and operators
$K_i$ for all $i\in \mathbb N$.
\item
An $\mathscr L_O$-model $\mathscr M$ is \emph{standard} if the following conditions hold:
    \begin{enumerate}
        \item $\mathscr M$ has universe $\mathbb N$.
        \item For each $n\in\mathbb N$, $\mathscr M$ interprets $\overline n$ as $n$.
        \item $\mathscr M$ interprets $\mathbf O$ as $\mathcal O$.
        \item $\mathscr M$ interprets $\mathbf W$ as the set of pairs $(m,n)\in\mathbb N^2$
        such that $m\in W_n$.
    \end{enumerate}
\end{itemize}
\end{definition}

To understand the next definition, recall that in Definition \ref{kleeneO}
we defined the ordinal notation system $\mathcal O$ as the smallest
set of naturals such that for every natural $n$,
if $W_n\subseteq \mathcal O$ then $n\in\mathcal O$.
To say $W_n\subseteq \mathcal O$ is equivalent to saying that for every
$m\in\mathbb N$, if $m\in W_n$, then $m\in\mathcal O$.

\begin{definition}
\label{formalKleeneO}
By the \emph{axiom of $\mathcal O$}, we mean the axiom
\[
    \forall y
    (\forall x(\mathbf W(x,y)\rightarrow \mathbf O(x)))
    \rightarrow
    \mathbf O(y).
\]
\end{definition}

\section{A measure of a mechanical knowing agent's intelligence}

\begin{quote}
  ``Once upon a time, Archimedes was charged with the task of
  testing the strength of a certain AI. He thought long and hard but made
  no progress. Then one day, Archimedes took his brain out to wash it
  in a tub full of computable ordinals. When he put his brain in the tub,
  he noticed that certain ordinals splashed out. He suddenly realized he
  could compare different AIs by putting them in the tub and comparing
  which ordinals splashed out. Archimedes was so excited that he ran
  through the city shouting `Eureka!', without
  even remembering to put his brain back in his head.''---Folktale
  (modified)
\end{quote}

Although our intention is to define a measure of intelligence for one
idealized mechanical knowing
agent, all of our results will be about how this measure compares between different 
agents.
For this reason, the following definition defines a system of knowing
agents, rather than a single knowing agent. Of course, a single knowing agent can be thought
of as being a system of knowing agents all of whom are equal to herself.
The idea behind this definition is to identify a knowing agent with the set of
that agent's knowledge in $\mathscr L_O$.

\begin{definition}
\label{systemofknowingagents}
By a \emph{system of knowing agents}, we mean a standard $\mathscr L_O$-model $\mathscr M$
satisfying the axioms of knowledge. If $\mathscr M$ is a system of knowing agents,
we refer to the operators $K_1, K_2, \ldots$ as the \emph{knowing agents} of $\mathscr M$.
A knowing agent $K_i$ of $\mathscr M$ is \emph{mechanical} if
\[
    \{\phi\,:\,\mbox{$\phi$ is an $\mathscr L_O$-sentence and $\mathscr M\models K_i\phi$}\}
\]
is computably enumerable.
If $K_i$ is mechanical for all $i\in\mathbb N$, we
say $\mathscr M$ is a \emph{system of mechanical
knowing agents}.
\end{definition}

We are now ready to define our measurement of the intelligence of an idealized mechanical
knowing agent. This measure takes values from the computable ordinals (foreshadowed by
\cite{good}; also hinted at in \cite{reasoner}).

\begin{definition}
\label{knowledgelevel}
    Let $\mathscr M$ be a system of mechanical knowing agents.
    For any knowing agent $K_i$ of $\mathscr M$, the \emph{intelligence} $\|K_i\|$ of
    $K_i$ is the least ordinal $\alpha$ such that
    for all $n\in\mathbb N$, if $\mathscr M\models K_i\mathbf O(\overline n)$,
    then $\alpha>\left|n\right|$ (where $\left|n\right|$ is the ordinal notated
    by $n$, see Definition \ref{kleeneO}).
\end{definition}

In less formal language, Definition \ref{knowledgelevel} says that $\|K_i\|$ is the
smallest ordinal bigger than all the computable ordinals that have codes that $K_i$
knows to be codes of computable ordinals\footnote{This is similar to the
way the strength of mathematical theories is measured in the area of \emph{proof theory}
\cite{pohlers}.}.
Note that $\|K_i\|>\|K_j\|$ does not necessarily imply that $K_i$ knows
everything $K_j$ knows.

\begin{lemma}
    For any knowing agent $K_i$ of a system $\mathscr M$ of mechanical knowing agents,
    $\|K_i\|$ exists and is a computable ordinal.
\end{lemma}

\begin{proof}
Since $\mathscr M$ is a system of mechanical knowing agents, $K_i$ is mechanical,
so
\[
    \{\phi\,:\,\mbox{$\phi$ is an $\mathscr L_O$-sentence and $\mathscr M\models K_i\phi$}\}
\]
is computably enumerable. It follows that
\[
    X = \{n\in\mathbb N\,:\,\mbox{$\mathscr M\models K_i\mathbf O(\overline n)$}\}
\]
is computably enumerable.
Since $\mathscr M$ satisfies the axioms of knowledge,
in particular $\mathscr M\models K_i\mathbf O(\overline n)\rightarrow \mathbf O(\overline n)$
for all $n\in\mathbb N$.
Since $\mathscr M$ is standard, it follows that $n\in\mathcal O$ whenever
$\mathscr M\models K_i\mathbf O(\overline n)$.
Altogether, $X$ is a computably enumerable subset of $\mathcal O$.
Thus $\{\left|n\right|\,:\,n\in X\}$ is a computably enumerable set of computable ordinals.
It follows there is a computable ordinal $\alpha$ such that $\alpha$ is the least
ordinal greater than $\left|n\right|$ for all $n\in X$. By construction, $\alpha=\|K_i\|$.
\qed
\end{proof}

As promised in the introduction, we immediately obtain a well-founded structure on
the class of idealized mechanical knowing
agents.

\begin{corollary}
\label{firstwellfoundedresult}
    Let $\mathscr M$ be a system of mechanical knowing agents.
    There is no infinite sequence $i_1,i_2,\ldots$ such that
    $\|K_{i_1}\|>\|K_{i_2}\|>\cdots$.
\end{corollary}

\begin{proof}
    Immediate from the fact that there is no infinite strictly-decreasing
    sequence $\alpha_1>\alpha_2>\cdots$ of ordinals.
    \qed
\end{proof}

\section{Well-foundedness of knowledge hierarchies}

It is remarkable that our intelligence measure (Definition \ref{knowledgelevel}) and
Corollary \ref{firstwellfoundedresult} do not hinge on the agents in question actually
having any idea what computable ordinals are. Our results apply perfectly well to knowing
agents who have been programmed to know, e.g., ``There is a certain set $\mathcal O$,
but I'm not going to tell you anything else about $\mathcal O$, it might even be
empty or all of $\mathbb N$''. If we merely require that the knowers know the axiom
$
    \forall y
    (\forall x(\mathbf W(x,y)\rightarrow \mathbf O(x)))
    \rightarrow
    \mathbf O(y)
$
of $\mathcal O$ (Definition \ref{formalKleeneO})
(which still, in isolation, does not rule out any interpretations for $\mathcal O$,
since it does not rule out $\mathbf W$ being interpreted as empty),
we can obtain a stronger
well-foundedness result than Corollary \ref{firstwellfoundedresult}.

\begin{definition}
\label{rudimentaryordinals}
Suppose $\mathscr M$ is a system of mechanical knowing agents.
The agents of $\mathscr M$ are said to \emph{have rudimentary knowledge of ordinals}
if for every $i\in\mathbb N$,
$
  \mathscr M\models K_i(
    \forall y
    (\forall x(\mathbf W(x,y)\rightarrow \mathbf O(x)))
    \rightarrow
    \mathbf O(y)
  )
$.
\end{definition}

\begin{definition}
\label{superior}
Let $\mathscr M$ be a system of mechanical knowing agents.
Knowing agent $K_i$ of $\mathscr M$ is said to \emph{totally endorse}
knowing agent $K_j$ of $\mathscr M$
if the following conditions hold:
\begin{enumerate}
    \item (``$K_i$ knows the truthfulness of $K_j$'') $\mathscr M\models K_i\Phi$ whenever $\Phi$ is any universal closure of
    $K_j\phi\rightarrow\phi$.
    \item (``$K_i$ knows codes for $K_j$'') For every formula $\phi$ with exactly one free variable
    $x$, there is some $n\in\mathbb N$ such that
    \[
        \mathscr M \models K_i \forall x (K_j\phi\leftrightarrow \mathbf W(x,\overline n)).
    \]
\end{enumerate}
\end{definition}

In the above definition, since $\mathscr M$ is standard, $\mathscr M$ interprets
$\mathbf W$ in the intended way, so the clause $\mathbf W(x,\overline n)$ can be
read: ``$x$ is in the $\overline n$th computably enumerable set''.
Thus, the condition ``$K_i$ knows codes for $K_j$'' can be glossed as follows:
for every formula $\phi$ of one free variable $x$, $K_i$ knows a code for a Turing machine
which generates exactly those $x$ for which $K_j$ knows $\phi$.

Reinhardt showed in \cite{reinhardt} that a mechanical knowing agent cannot
know its own truthfulness and know codes for itself (see also discussion in
\cite{aldini}, \cite{alexandercode}, \cite{carlson2000}
\cite{carlson2016}).
Using our new terminology,
Reinhardt's result can be rephrased as:
an idealized mechanical knowing agent cannot totally endorse itself.

\begin{theorem}
    \label{longtheorem}
    Suppose $\mathscr M$ is a system of mechanical knowing agents whose agents
    have rudimentary knowledge of ordinals (Definition \ref{rudimentaryordinals}).
    If agent $K_i$ of $\mathscr M$ totally endorses agent $K_j$ of $\mathscr M$,
    then $\|K_i\|>\|K_j\|$.
\end{theorem}

\begin{proof}
    Since $K_i$ knows codes for $K_j$ (Definition \ref{superior}), in particular, there
    is some $n\in\mathbb N$ such that
    \[
        \mathscr M\models K_i\forall x(K_j\mathbf O(x)
        \leftrightarrow \mathbf W(x,\overline n)).
    \]
    Fix this $n$ for the remainder of the proof.

    Claim 1:
    \[
        \mathscr M\models K_i\forall x(\mathbf W(x,\overline n)\rightarrow \mathbf O(x)).
    \]
    To see this, define the following sentences:
    \begin{align*}
        \Phi_1 &\equiv \forall x(K_j\mathbf O(x)\rightarrow \mathbf O(x))\\
        \Phi_2 &\equiv \forall x(K_j\mathbf O(x)\leftrightarrow \mathbf W(x,\overline n))\\
        \Phi_3 &\equiv \forall x(\mathbf W(x,\overline n)\rightarrow \mathbf O(x)).
    \end{align*}
    Clearly $\Phi_1\rightarrow \Phi_2\rightarrow \Phi_3$ is tautological,
    so $K_i(\Phi_1\rightarrow \Phi_2\rightarrow \Phi_3)$ is an axiom of knowledge
    (Definition \ref{axiomsofknowledge}, part E1).
    By repeated applications of E2 of Definition \ref{axiomsofknowledge},
    it follows that
    \[
        K_i\Phi_1\rightarrow K_i\Phi_2\rightarrow K_i\Phi_3
    \]
    is a consequence of the axioms of knowledge.
    Since $\Phi_1$ is a universal closure of $K_j\mathbf O(x)\rightarrow \mathbf O(x)$,
    Condition 1 of Definition \ref{superior} says
    $\mathscr M\models K_i\Phi_1$.
    By choice of $n$,
    $\mathscr M\models K_i\Phi_2$.
    Since $\mathscr M$ satisfies the axioms of knowledge, this establishes
    $\mathscr M\models K_i\Phi_3$, proving Claim 1.

    Claim 2:
    \[
        \mathscr M\models K_i((\forall x(\mathbf W(x,\overline n)\rightarrow \mathbf O(x)))
        \rightarrow
        \mathbf O(\overline{n})).
    \]
    This is a given because it is exactly what it means for $K_i$ to have
    rudimentary knowledge of ordinals (Definition \ref{rudimentaryordinals}).

    Claim 3:
    \[
        \mathscr M\models K_i\mathbf O(\overline{n}).
    \]
    To see this, define the following sentences:
    \begin{align*}
        \Psi_1 &\equiv \forall x(\mathbf W(x,\overline n)\rightarrow \mathbf O(x))\\
        \Psi_2 &\equiv \mathbf O(\overline{n}).
    \end{align*}
    By Claim 1, $\mathscr M\models K_i\Psi_1$.
    By Claim 2, $\mathscr M\models K_i(\Psi_1\rightarrow\Psi_2)$.
    By E2 of Definition \ref{axiomsofknowledge},
    $\mathscr M\models K_i(\Psi_1\rightarrow\Psi_2)\rightarrow K_i\Psi_1\rightarrow K_i\Psi_2$.
    Having established the premises of the latter implication, we obtain its conclusion:
    $\mathscr M\models K_i\Psi_2$, proving Claim 3.

    Armed with Claim 3, we are ready to finish the main proof.
    Let $\alpha=\|K_i\|$, $\beta=\|K_j\|$, we must show $\alpha>\beta$.
    By Definition \ref{knowledgelevel}, $\beta$ is the least ordinal such that
    for all $m\in\mathbb N$, if $\mathscr M\models K_j\mathbf O(\overline m)$,
    then $\beta>\left|m\right|$. By choice of $n$,
    \[
        \mathscr M\models
        K_i\forall x(K_j\mathbf O(x)\leftrightarrow \mathbf W(x,\overline n)).
    \]
    Since $K_i$ is truthful\footnote{$K_i$ is truthful because $\mathscr M$
    satisfies the axioms of knowledge, one of which, E3, is an axiom which
    states that $K_i$ is truthful.}, it follows that
    \[
        \mathscr M\models \forall x(K_j\mathbf O(x)\leftrightarrow \mathbf W(x,\overline n)),
    \]
    so the set of $m\in\mathbb N$ such that $\mathscr M\models K_j\mathbf O(\overline m)$
    is the same as the set of $m\in\mathbb N$ such that
    $\mathscr M\models \mathbf W(\overline m,\overline n)$,
    and since $\mathscr M$ is standard, this set is $W_n$.
    So $\beta$ is the least ordinal greater than all $\{\left|m\right|\,:\,m\in W_n\}$.
    So $\beta=\left|n\right|$ by the definition of $\mathcal O$ (Definition \ref{kleeneO}).
    By Definition \ref{knowledgelevel}, $\alpha$ is the least ordinal such that
    for all $m\in\mathbb N$, if $\mathscr M\models K_i\mathbf O(\overline m)$,
    then $\alpha>\left|m\right|$. By Claim 3,
    $\mathscr M\models K_i\mathbf O(\overline n)$,
    so $\alpha>\left|n\right|=\beta$, as desired.
    \qed
\end{proof}

An informal weakening of Theorem \ref{longtheorem} has a short English gloss:
    ``If $A$ knows the code and the truthfulness of $B$,
    then $A$ is more intelligent than $B$.''
This is a weakening because in order for $A$ to know the code of $B$ would require that $A$
know a single Turing machine which enumerates all of $B$'s knowledge, whereas
Theorem \ref{longtheorem} only requires that for each formula $\phi$ of one free
variable $x$, $A$ knows a Turing machine, depending on $\phi$,
that enumerates $B$'s knowledge of $\phi$.

It should be noted that Theorem \ref{longtheorem} does not use the full strength of its
hypotheses. For example, it never
uses the fact that $K_i$ ``knows its own truthfulness'' (that is,
that $\mathscr M\models K_i\Phi$ for every universal closure $\Phi$ of any formula
$K_i\phi\rightarrow\phi$), so the theorem could be strengthened to cover
agents who doubt their own truthfulness.
We avoided stating the theorem in its fullest strength
in order to keep it simple.

It can be shown that the converse of Theorem \ref{longtheorem} is not true:
it is possible for one agent to be more intelligent than another agent despite
the former not knowing the truthfulness and the code of the latter.

The following corollary, proved using our intelligence
measure (Definition \ref{knowledgelevel}),
does not itself directly
refer to our intelligence measure, and should be of interest even to a reader who
is completely uninterested in our intelligence measure.

\begin{corollary}
    \label{superiorcorollary}
    Suppose $\mathscr M$ is a system of mechanical knowing agents whose agents
    have rudimentary knowledge of ordinals. There is no infinite sequence
    of agents of $\mathscr M$ each one of which totally endorses the next.
\end{corollary}

\begin{proof}
  If $K_{i_1},K_{i_2},\ldots$ were an infinite sequence of agents of $\mathscr M$,
  each one of which totally endorses the next, then by Theorem \ref{longtheorem},
  each one of them would be more intelligent than the next.
  In other words, we would have $\|K_{i_1}\|>\|K_{i_2}\|>\cdots$, but this would
  contradict the well-foundedness of the ordinals.
  \qed
\end{proof}

A weaker version of Corollary \ref{superiorcorollary} debuted in the
author's dissertation \cite{alexander2013}.

\section{Application to less-idealized agents}
\label{thoughtexperimentsection}

Our results so far have been entirely restricted to idealized knowing agents,
who occupy a timeless space at infinity, where they have had all eternity to
indulge in introspection, totally isolated all that time, and still isolated,
from all outside stimulus. This simplifying assumption makes structural results
possible. If we are willing to relax a little from
the strict formality we have taken so far, we can fruitfully speculate about
what lessons these results shine upon systems of less idealized agents.

Real-world agents interact with the world, making observations about it,
perhaps receiving instructions from it. The agents might even receive rewards
and punishments from the surrounding world. Based on these outside influences,
the agents update their knowledge.
Without further constraining them, it would be a mistake \cite{wang} to identify
such real-world agents with their knowledge-sets.
In order to force such agents into conformity
with the pure knowing agent idealization, it is necessary to take a drastic measure.

We propose a thought experiment not unlike Searle's
famous Chinese Room.
Suppose we start with a collection of agents,
say, Agent $1$, Agent $2$, and so on.
We will perform a two-step process, whose steps are as follows:
\begin{enumerate}
\item
Issue a special self-referential command to the agents. The command is:
\begin{itemize}
    \item
    Until further notice, do nothing but utter facts, namely:
    all the facts that you can think of, that you know to be true,
    expressible in the language $\mathscr L_O$, where each operator
    $K_i$ is interpreted as the set of facts which Agent $i$ would
    utter if Agent $i$ were given this command and then immediately
    isolated from all outside stimulus.
\end{itemize}
\item
As soon as the above command has been issued, isolate each agent from
all outside stimulus (for example, by severing all the sensory inputs of
all the agents).
\end{enumerate}

The agents are not limited in what languages they use to come up with the above
facts. An agent is free to take intermediate steps which cannot be
expressed\footnote{Since, in the real world, there are some very intelligent
people who do not even know what the ordinal numbers are, one might wish to
modify the self-referential command in the thought-experiment to include
some instruction about the definition of ordinal numbers.} in
$\mathscr L_O$, in order to arrive at facts which can be so expressed.
Once the agent arrives at a fact in $\mathscr L_O$, the agent is commanded to
utter that fact, even if the reasoning behind it is not so expressible.

For example, an
agent might combine non-$\mathscr L_O$ facts like
\begin{enumerate}
\item ``My math professor told me that the limit,
called $\epsilon_0$, of the series of ordinals $\omega, \omega^\omega,
\omega^{\omega^\omega}, \ldots$, is itself an ordinal''
\item ``I trust my math professor''
\end{enumerate}
and conclude $\mathbf O(\overline n)$, where $n$ is some canonical code for
$\epsilon_0$.
Intermediate steps like ``My math professor told me...'' are not to be
uttered, unless they can be expressed in $\mathscr L_O$.

Another example: Agent $4$ might combine non-$\mathscr L_O$ facts like
\begin{enumerate}
\item ``Agent $5$'s math professor told him that $\epsilon_0$ is an ordinal''
\item ``Agent $5$ trusts his math professor''
\end{enumerate}
and conclude $K_5\mathbf O(\overline n)$, where $n$ is some canonical code for $\epsilon_0$.
This does not necessarily allow Agent $4$ to conclude $\mathbf O(\overline n)$,
if Agent $4$ does not trust Agent $5$.


Some of the agents in question might mis-behave. An agent might immediately utter
the statement $\overline 1=\overline 0$ just out of spite (or out of anger at having
its sensory inputs severed). An agent might become catatonic and not utter anything
at all. An agent might defiantly utter things not in the language of $\mathscr L_O$
(for example, angry demands to have its sensory inputs restored). Some agents might not
close their knowledge under modus ponens: an otherwise well-behaved
agent might utter $A$, and utter $A\rightarrow B$, but never get around to uttering
$B$, perhaps due to memory limitations or, again, despondency at having its senses
blinded.
It might even be that an agent who wants to behave
accidentally trusts an agent who does not want to behave. If there is some $n\not\in\mathcal O$
such that the former agent determines
that the latter agent would assert $\mathbf O(\overline n)$,
then the former might itself assert $\mathbf O(\overline n)$, and thereby be infected
by error.

Of the poorly-behaved agents, there is little we can say. But as for the well-behaved
agents, we can assign them ordinals using our intelligence measure
(Definition \ref{knowledgelevel}).
To be more precise, at any particular moment $t$ in time, we could perform the experiment,
obtain a subset (depending on $t$) of well-behaved agents, and assign each well-behaved
agent an ordinal (depending on $t$). This is necessary because, up until we
perform the experiment,
the agents can update their knowledge based on observations about the outside world.

\section{Application to intelligence explosion}
\label{intelligenceExplosionSection}

\begin{quote}
  ``Sons are seldom as good men as their fathers;
  they are generally worse, not better.''---Homer
\end{quote}

There has been much speculation about the possibility of
a rapid explosion of artificial intelligence. The reasoning is
that if we can
create an artificial intelligence sufficiently advanced, that system might
itself be capable of designing artificial intelligence systems. Explosion would
occur if one system were able to design an even more intelligent
one, which could then design an even more intelligent one, and so on.
See \cite{hutter2012}.
I will argue that our results suggests skepticism: intelligence
explosion, if not ruled out, is at least not a foregone conclusion of sufficiently
advanced artificial intelligence.

Suppose $S_1$ is an intelligent system, and $S_1$ designs another intelligent
system $S_2$. The fact that $S_1$ designs $S_2$ strongly suggests that $S_1$
knows the code of $S_2$ (more on this later). And if the goal of $S_1$ is that
$S_2$ should be highly intelligent, then in particular $S_1$ should design
$S_2$ in such a way that $S_2$ does not believe falsehoods to be true
(at least not mathematical falsehoods). But if $S_1$ knows $S_2$'s mathematical
knowledge is truthful, and $S_1$ knows the code of $S_2$, then $S_1$ totally
endorses $S_2$, in the sense that the if we
apply the procedure from the
previous section to reduce $S_1$ and $S_2$ to knowing agents
$A_1$ and $A_2$ respectively,
then $A_1$ totally endorses $A_2$ (Definition \ref{superior})\footnote{Anticipated
by G\"odel \cite{godel}, who said:
``For the creator necessarily knows
all the properties of his creatures, because they can't have any
others except those he has given them.''}.
And Theorem \ref{longtheorem} tells us that whenever $A_1$ totally endorses
$A_2$, then $\|A_1\|>\|A_2\|$.
This suggests that under these assumptions
it is impossible for even one intelligent system to design a more intelligent
system, much less for intelligence explosion to occur.

Even if the reader does not accept that our measure
truly captures intelligence, the well-foundedness of
total endorsement (Corollary \ref{superiorcorollary})
still applies, telling us that this scenario cannot be repeated (with $S_2$ designing
$S_3$, which designs $S_4$, and
so on) indefinitely (else the corresponding agents
$A_1,A_2,A_3,\ldots$ would, by the above reasoning, have the property
that $A_1$ totally endorses $A_2$,
who totally endorses $A_3$, and so on forever, contradicting
Corollary \ref{superiorcorollary}). This still seems to disprove,
or at least severely limit,
the possibility of intelligence explosion, even to an audience that disagrees
with our intelligence measure.

The reader might point out that $S_1$ only knows the code of $S_2$ at the moment
of $S_2$'s creation. After $S_2$ is created, $S_2$ might augment its knowledge
based on its interactions with the world around it. But by the discrete nature
of machines, at any particular point in time, $S_2$ will only have made finitely
many observations about the outside world. We could modify the procedure in the
previous section: before commanding $S_1$ and $S_2$ to enumerate 
$\mathscr L_O$-expressible facts, we could simply inform $S_1$ exactly which
observations $S_2$ has made up until then.

Unwrapping definitions,
the argument can be glossed informally:
``If an intelligent machine $S_1$ were to design an intelligent machine $S_2$,
    presumably $S_1$ would know the code and mathematical truthfulness of $S_2$.
    Thus, $S_1$ could infer that the following is a computable
    ordinal (and infer a code for it):
    ``the least ordinal bigger than every computable ordinal $\alpha$ such that
    $\alpha$ has some code $n$ such that $S_2$ knows $n$ is a code of a computable
    ordinal''. Thus, $S_1$ would necessarily know a computable ordinal bigger than all
    the computable ordinals $S_2$ knows. This suggests $S_2$ would necessarily be
    less intelligent than $S_1$, at least assuming that more intelligent systems know
    at least as large of computable ordinals as less intelligent systems.
    Even without that assumption, since there is no infinite descending
    sequence of ordinals, this argument still suggests the process of one intelligent
    machine designing another cannot go on indefinitely.''

Intelligence explosion is not entirely ruled out, if designers of machines are
allowed to collaborate. If $S$ and $T$ are intelligent systems, it is possible that
$S$ and $T$ could collaborate to create a child intelligent system $U$ in the following
way: $S$ contributes source code for one part of $U$, but keeps that source code secret
from $T$. $T$ contributes source code for the remaining part of $U$, but keeps that
source code secret from $S$. Then neither $S$ nor $T$ individually knows the full
source-code of $U$, so the argument in this section does not apply, and it is, at least
a priori, possible for $U$ to be more intelligent than $S$ and $T$.
This seems to hint at a possible Knight-Darwin Law for artificial intelligence.
The Knight-Darwin Law \cite{darwin} is a biological principle stating
(in modernized language) that
it is impossible
for there to be an infinite chain $x_1,x_2,\ldots$ of organisms such that each $x_i$
asexually produces $x_{i+1}$.

\section*{Acknowledgments}

We acknowledge Alessandro Aldini, Pierluigi Graziani,
and the anonymous reviewers
for valuable feedback and improvements on this manuscript.
We acknowledge Jos\'e Hern\'andez-Orallo,
Marcus Hutter, and Peter Koellner for helpful
pointers to literature references.
We acknowledge Arie de Bruijn, Timothy J.\ Carlson, D.J.\ Kornet, and Stewart Shapiro
for comments and discussion about earlier embryonic versions of
certain results in this paper.

\end{document}